\documentclass{article}
\pdfoutput=1

\usepackage[preprint]{neurips_2026}

\usepackage[utf8]{inputenc}
\usepackage[T1]{fontenc}
\usepackage{hyperref}
\hypersetup{
  pdftitle={On the Representational Geometry of Dynamic Programs},
  pdfauthor={Richard F. M. Lim, Ruriko Yoshida}
}
\usepackage{booktabs}
\usepackage{array}
\usepackage{amssymb}
\usepackage{microtype}
\usepackage{xcolor}
\usepackage{mathtools}
\usepackage{amsthm}
\usepackage{tikz}
\usetikzlibrary{arrows.meta,calc,positioning}

\newtheorem{theorem}{Theorem}
\newtheorem{proposition}[theorem]{Proposition}
\newtheorem{lemma}[theorem]{Lemma}
\newtheorem{corollary}[theorem]{Corollary}
\newtheorem{definition}[theorem]{Definition}
\newtheorem{example}[theorem]{Example}

\DeclareMathOperator{\dist}{dist}
\DeclareMathOperator{\vmin}{vmin}

\title{On the Representational Geometry of Dynamic Programs}

\author{
  Richard F. M.~Lim\textsuperscript{1, 3}\\
  \texttt{rlim@bowdoin.edu}
  \And
  Ruriko~Yoshida\textsuperscript{2, 3}\\
  \texttt{ry@math.aau.dk}
  \AND
  \normalfont
  \textsuperscript{1} Departments of Mathematics and Computer Science, Bowdoin College, Brunswick, Maine \\
  \textsuperscript{2} Department of Mathematical Sciences, Aalborg University, Denmark \\
  \textsuperscript{3} Department of Operations Research, Naval Postgraduate School, Monterey, California
}

\begin{document}

\maketitle

\begin{abstract}
Standard neural architectures often fail to generalize to longer inputs for dynamic programming (DP) targets. We investigate what makes this hard geometrically. Every finite min-plus DP is a shortest path on a DAG, which is equivalently a tropical polynomial whose extended Newton polyhedron encodes the decision boundary of which path wins.
We prove these three descriptions (graph, polynomial, polyhedron) form isomorphic semirings at two levels --- formal polynomials and their computed functions --- connected by operations that characterize all structural redundancies. We then address the length-generalization question geometrically: does the decision boundary at length $T$ decide the boundary at $T+1$? We present two structural negatives. The semiring's two native ways to reduce dimension (setting a variable to each identity) are neither injective nor always closed within the DP. Series and parallel composition fail to construct all DAG topologies from smaller sub-DAGs, and even all terminal-only operations do not capture all DP compositions.
\end{abstract}

\section{Introduction}\label{sec:intro}

A dynamic program is defined by a recurrence that does not depend on the size of its input, so a model that has truly learned the recurrence would run at any length.
Standard sequence architectures do not. Trained on combinatorial optimization targets, they fit the training lengths and then degrade sharply beyond them, a failure documented across recurrent and attention-based models and robust to scale \citep{deletang2023,anil2022}.
Tropical Attention \citep{hashemi2025} is a conspicuous exception; replacing addition and multiplication with max-plus (tropical) operations $(\max,+)$, it achieves strong length out-of-distribution (OOD) generalization on classic DP problems.

The premise is classical. After unrolling, every finite min-plus dynamic program is: (1) a single-pair shortest-path problem on a decision DAG with a single source $s$ and single sink $z$, which we call \emph{two-terminal} \citep{bellman1957,gondran1979}; and (2) a tropical polynomial \citep{joswig2021essentials}.
For alignment with the shortest-path formulation, we work throughout in the dual $(\min, +)$ semiring $\mathbb{T} = \mathbb{R} \cup \{+\infty\}$, equivalent to $(\max, +)$ under negation \citep{joswig2021essentials}.
We write $a \oplus b = \min(a,b)$ and $a \odot b = a + b$ when emphasizing the semiring structure.

Our work is focused on length generalization: a learner observes small graphs with at most $T$ variables and must predict every graph after that. We therefore investigate whether this is possible via the questions: (1) What algebraic structure in DP allows algorithms to work on every length? (2) Does this structure allow us to infer larger instances from smaller ones?

Composing two-terminal DAGs in series adds their path lengths; composing them in parallel takes the minimum (Figure~\ref{fig:sp-comp} in Appendix~\ref{app:sp-comp}).
We first formalize the relationship between graphs, polynomials, and geometry via a triple semiring isomorphism (Section~\ref{sec:triality}). Multiple \emph{formal} polynomials may map to the same \emph{function} --- $\min(0, x, 2x) = \min(0, 2x)$ for all $x$ despite different supports --- so we construct the formalism at both levels and relate them with congruence quotients. Graph structures are even richer than their polynomials, so we provide a canonical \emph{deshared} normal form for every tropical polynomial and a list of invariant operations that traverse equivalence classes on DAGs (Section~\ref{sec:invariance}).

A natural way to relate larger instances to smaller ones is substitution at the semiring
identities: setting a variable to $+\infty$ removes it from a shortest-path or knapsack
instance, while setting it to $0$ removes it from a min-subarray or assignment instance. We employ the above framework to provide a characterization of length recovery under substitution in terms of graph, polynomial, polyhedron, and the \emph{decision boundary} $\mathcal{T}(f)$, known as the \emph{tropical hypersurface} \citep{zhang2018tropical}. We show in Section~\ref{sec:geometry} that both substitutions are not always closed within the DP, and do not recover length even when they are: the $T$-th
instance does not determine the $(T{+}1)$-th.  Finally, Section~\ref{sec:sp-limits} identifies series and parallel composition as the two
faces of Bellman's equation on decision DAGs, then exposes a two-layered topological
insufficiency: together they do not generate all topologies of terminal-only operations, and terminal-only operations do not capture all DP compositions.

\section{The triality: one problem, three languages}\label{sec:triality}
Appendix~\ref{app:crash-course} details the tropical geometry background used throughout.
A \emph{formal tropical polynomial} in $d$ variables is a finite tropical sum of monomials,
$f(x) = \bigoplus_{\alpha} c_\alpha \odot x^{\odot\alpha}
      = \min_\alpha(c_\alpha + \langle \alpha, x\rangle)$,
with exponents $\alpha \in \mathbb{Z}_{\ge 0}^d$ and coefficients $c_\alpha \in \mathbb{T}$
($c_\alpha = {+}\infty$ meaning the term is absent). Its \emph{support} is
$S(f) = \{\alpha : c_\alpha \ne {+}\infty\}$, and its \emph{lifted support}
$\hat S(f) = \{(\alpha, c_\alpha) : \alpha \in S(f)\}$ raises each point to its
coefficient height. Terms sharing an exponent keep only the lesser coefficient,
$(c \oplus c') \odot x^{\odot\alpha}$; call this per-exponent reduction $\vmin$ (vertical
minimum), and say $f$ is in \emph{normal form} once reduced by it, one coefficient per exponent.
The normal-form lifted supports $\mathcal{L}_d$ --- finite subsets of
$\mathbb{Z}_{\ge 0}^d \times \mathbb{R}$ with at most one point over each exponent --- form a
commutative idempotent semiring under $\vmin(\cdot \cup \cdot)$ and Minkowski addition
(Proposition~\ref{prop:config-semiring}). The \emph{extended Newton polyhedron} is
$\Gamma(f) = \operatorname{conv}(\hat S(f)) + \mathbb{R}_{\ge 0}\,e_{d+1}$, the convex hull of
the lifted support extended upward. Two polynomials compute the same function if and only if
their extended Newton polyhedra agree \citep{joswig2021essentials}. Thus, convexity is the second quotient, and $\vmin$ becomes $\operatorname{conv}$. The ordinary \emph{Newton polytope}
$\mathcal{N}(f) = \operatorname{conv}(S(f))$ is its vertical projection, constructed via monomials but not their coefficients.

Every two-terminal weighted DAG $G$ computes a min-plus \emph{path polynomial}. This may be constructed by enumerating paths $\pi$ from source $s$ to sink $z$ as follows: $f_G = \min_{\pi : s \to z} \sum_{e \in \pi} w_e$, the minimum over $s$--$z$ paths in
normal form. For DAGs $A$ and $B$, $f_{A;B} = f_A + f_B$ and $f_{A\|B} = \min(f_A, f_B)$ (Proposition~\ref{prop:path-hom}). However, two graphs may have different topologies but compute exactly the same shortest-path function, so we define equivalence relations that erase exactly this excess:
$G \approx_{\mathrm f} G' $ iff $ \hat{S}(f_G) = \hat{S}(f_{G'}), $ and $ G \approx_{\mathrm v} G' $ iff $ f_G, f_{G'} \text{ compute the same function.}$

Both are congruences (compatible with $;$ and $\|$). The following
theorem says each quotient produces a semiring isomorphic to its associated polynomial and
polyhedron: at each level, graph, algebra, and geometry are the \emph{same} semiring.

\begin{theorem}\label{thm:triality}
The following diagram of commutative idempotent semirings commutes;
horizontal maps are isomorphisms, vertical maps are surjective homomorphisms.
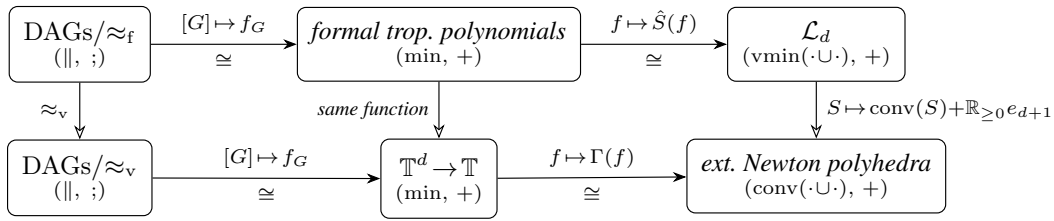
\begin{center}
\resizebox{\linewidth}{!}{%
\begin{tikzpicture}[
  node distance=0.7cm and 1.8cm,
  every node/.style={font=\small},
  sr/.style={draw,rounded corners=3pt,inner sep=5pt,align=center},
  iso/.style={->},
  surj/.style={-{Stealth[open,length=5pt,width=4pt]}},
  >=Stealth]
\node[sr] (Gf)
  {$\mathrm{DAGs}/{\approx_{\mathrm f}}$\\[-2pt]
   {\scriptsize$(\|,\;{;})$}};
\node[sr,right=of Gf] (Af)
  {formal trop.\ polynomials\\[-2pt]
   {\scriptsize$(\min,\,{+})$}};
\node[sr,right=of Af] (Sf)
  {$\mathcal{L}_d$\\[-2pt]
   {\scriptsize$(\vmin(\cdot\!\cup\!\cdot),\,{+})$}};
\node[sr,below=of Gf] (Gv)
  {$\mathrm{DAGs}/{\approx_{\mathrm v}}$\\[-2pt]
   {\scriptsize$(\|,\;{;})$}};
\node[sr,below=of Af] (Av)
  {$\mathbb{T}^d\!\to\!\mathbb{T}$\\[-2pt]
   {\scriptsize$(\min,\,{+})$}};
\node[sr,below=of Sf] (Sv)
  {ext.\ Newton polyhedra\\[-2pt]
   {\scriptsize$(\operatorname{conv}(\cdot\!\cup\!\cdot),\,{+})$}};
\draw[iso] (Gf) -- node[above]{\scriptsize$[G]\!\mapsto\! f_G$}
                    node[below]{\scriptsize$\cong$} (Af);
\draw[iso] (Af) -- node[above]{\scriptsize$f\!\mapsto\!\hat S(f)$}
                    node[below]{\scriptsize$\cong$} (Sf);
\draw[iso] (Gv) -- node[above]{\scriptsize$[G]\!\mapsto\! f_G$}
                    node[below]{\scriptsize$\cong$} (Av);
\draw[iso] (Av) -- node[above]{\scriptsize$f\!\mapsto\!\Gamma(f)$}
                    node[below]{\scriptsize$\cong$} (Sv);
\draw[surj] (Gf) -- node[left]{\scriptsize$\approx_{\mathrm v}$} (Gv);
\draw[surj] (Af) -- node[left]{\scriptsize same function} (Av);
\draw[surj] (Sf) -- node[right]
  {\scriptsize$S\!\mapsto\!\operatorname{conv}(S){+}\mathbb{R}_{\ge 0}e_{d+1}$} (Sv);
\end{tikzpicture}}
\end{center}
\end{theorem}

The proof is in Appendix~\ref{app:triality}.
The practical consequence is that a claim may be formulated and checked in
whichever of the three languages makes it easiest, then translated to the others.

\section{A complete calculus of invariant operations on DAG equivalences}\label{sec:invariance}

We now give the operations relating two $\approx$-equivalent graphs.
The \emph{deshared normal form} $D_{\mathrm f}(G)$ is the parallel bundle of chains
giving each monomial of $f_G$ its own $s$--$z$ path, edges in variable order: the
graph-language canonical form of the polynomial as a $\min$ of monomials. Mediant
deletion (dropping dominated chains) then gives the \emph{function normal form}
$D_{\mathrm v}(G)$. Four local operations, each realizing a semiring property
(Table~\ref{tab:moves-axioms} in Appendix~\ref{app:invariance}), carry any graph to these forms.

\begin{theorem}\label{thm:invariant-complete}
For all two-terminal DAGs $G, G'$:
\begin{itemize}
  \item[(i)] \emph{Reduction.} A finite sequence of formal moves carries $G$ to
    $D_{\mathrm f}(G)$, using one vertex split per path intersection
    (Lemma~\ref{lem:deshare-count}). Mediant deletions
    (Lemma~\ref{lem:mediant}) then carry $D_{\mathrm f}(G)$ to
    $D_{\mathrm v}(G)$.
  \item[(ii)] \emph{Canonicity.} $D_{\mathrm f}(G) = D_{\mathrm f}(G')$ iff
    $G \approx_{\mathrm f} G'$, and $D_{\mathrm v}(G) = D_{\mathrm v}(G')$ iff
    $G \approx_{\mathrm v} G'$; the chains of $D_{\mathrm v}(G)$ are exactly the vertices of
    $\Gamma(f_G)$. Both congruences are therefore decidable.
\end{itemize}
\end{theorem}

Connectivity follows: two graphs are joined by formal moves iff they are
$\approx_{\mathrm f}$, and by formal moves plus mediant deletions iff they are
$\approx_{\mathrm v}$, since each side can reduce to its normal form.
The complete argument is in Appendix~\ref{app:invariance}.

\section{Substitution geometry}\label{sec:geometry}

To relate lengths we need an operation on congruence-classes that is well-defined across
all three languages and reduces the variable count.  The semiring offers two identity substitutions: write $P_i$
for $x_i \mapsto +\infty$ (the $\min$-identity) and $Q_i$ for $x_i \mapsto 0$ (the
$+$-identity).  Over a field, two slices and subtraction recover affine dependence on a
coordinate; tropically there is no subtraction, so a slice can only restrict. Does the restriction suffice? We present two mechanisms of identification failure.

\begin{proposition}\label{prop:image-Pi}
Let $f_G$ be a tropical polynomial in $t+1$ variables. Writing $\pi_i$ for deletion of
coordinate $i$ and $\hat\pi_i$ for its height-preserving lift, each of $P_i$ and $Q_i$ acts by a
single transformation across the three languages, and $Q_i(f_G) \le P_i(f_G)$ pointwise (a face
lies inside the shadow):
\begin{center}\small
\renewcommand{\arraystretch}{1.15}
\setlength{\tabcolsep}{3pt}
\newcolumntype{L}[1]{>{\raggedright\arraybackslash}p{#1}}%
\begin{tabular}{@{}L{0.128\linewidth}L{0.40\linewidth}L{0.438\linewidth}@{}}
\toprule
 & $P_i$: coordinate \emph{face} & $Q_i$: coordinate \emph{shadow} \\
\midrule
algebra & the $\alpha_i{=}0$ monomials, coefficients intact
        & support $\pi_i(S(f_G))$, $c_\beta = \min\{c_\alpha : \pi_i(\alpha){=}\beta\}$ \\
polyhedron & $\Gamma(f_G) \cap \{\alpha_i{=}0\}$, a coordinate face
        & $\hat\pi_i(\Gamma(f_G))$, the projection along $\alpha_i$ \\
hypersurface & recession complex in direction $e_i$
        & slice at $x_i{=}0$ \\
graph & delete every $x_i$-edge: $[G \setminus E_i]$
        & reweight every $x_i$-edge to $0$, then merge \\
\bottomrule
\end{tabular}
\end{center}
\end{proposition}
The two graph rows carry the two routes of the opening paragraph: for $P_i$, ``delete every
$x_i$-edge'' descends to classes; for $Q_i$, ``reweight every $x_i$-edge to $0$, then merge''
is a rule on the deshared canonical form, since contraction on a general graph can create or
destroy an $s$--$z$ path. The hypersurface row reads the same operations on the tropical
hypersurface $\mathcal{T}(f_G)$, the decision boundary of the DP
(Section~\ref{sec:intro}; \citealp{zhang2018tropical}): $P_i$ restricts to
the strategies that ignore position $i$, while $Q_i$ slices through the boundary at $x_i = 0$.
Neither substitution is injective, and occasionally they are not \emph{closed}; in Figure~\ref{fig:fiber-Pi-preimages}, neither $P_i$ nor $Q_i$ sends the proper-subset instance to a proper-subset instance. The exact fiber is computed in Corollary~\ref{cor:fiber-Pi}
(Appendix~\ref{app:geometry}).

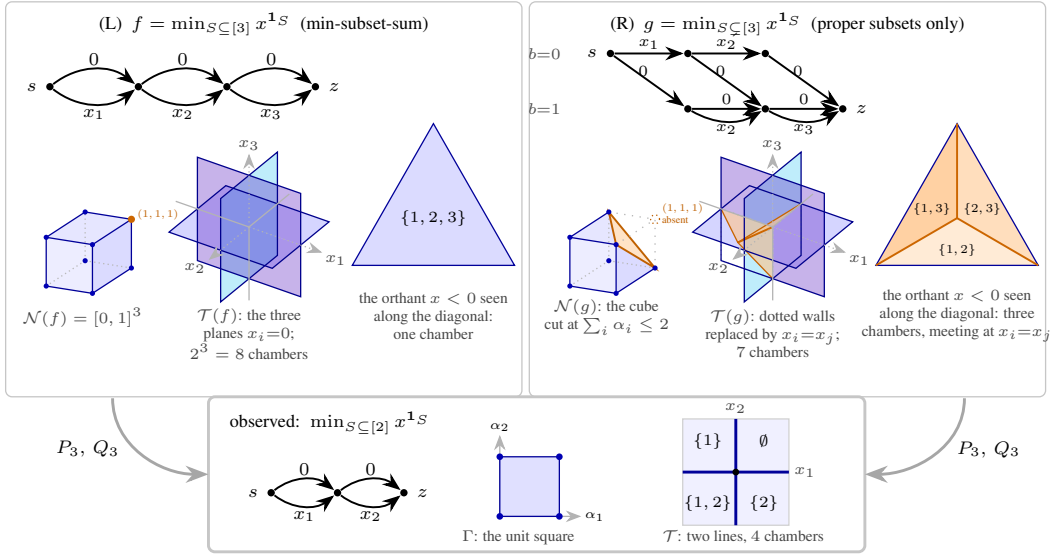
\begin{figure}[!t]
  \centering
\resizebox{\linewidth}{!}{%
\begin{tikzpicture}[>=Stealth, scale=0.80,
  every node/.style={font=\scriptsize,transform shape},
  cell/.style={draw=gray!45,rounded corners=2pt},
  vtx/.style={circle,fill=black,inner sep=1pt},
  ed/.style={->,semithick},
  pv/.style={fill=blue!70!black},
  cut/.style={fill=orange!80!black},
  hsplane/.style={opacity=0.30},
  hs3/.style={fill=blue!55},
  hs1/.style={fill=cyan!65},
  hs2/.style={fill=blue!45!violet},
  hsline/.style={draw=blue!55!black,thin},
  hsax/.style={draw=gray!60,thin},
  sub/.style={font=\tiny,gray!55!black,align=center}]
\begin{scope}[xshift=0cm]
  \draw[cell] (-0.35,-0.80) rectangle (6.65,4.55);
  \node[align=center] at (3.15,4.25)
    {(L) \ $f = \min_{S \subseteq [3]} x^{\mathbf 1_S}$ \ (min-subset-sum)};
  \node[vtx,label=left:{$s$}] (Ls) at (0.25,3.40) {};
  \node[vtx] (La) at (1.45,3.40) {};
  \node[vtx] (Lb) at (2.65,3.40) {};
  \node[vtx,label=right:{$z$}] (Lt) at (3.85,3.40) {};
  \foreach \u/\v/\lab in {Ls/La/{x_1},La/Lb/{x_2},Lb/Lt/{x_3}}{
    \draw[ed] (\u) to[bend left=38] node[above,inner sep=1.5pt]{$0$} (\v);
    \draw[ed] (\u) to[bend right=38] node[below,inner sep=1.5pt]{$\lab$} (\v);}
  \begin{scope}[shift={(0.20,0.70)},scale=0.66]
    \fill[blue!10] (0,1.00)--(0.95,0.85)--(1.75,1.37)--(0.80,1.52)--cycle;
    \fill[blue!14] (0.95,-0.15)--(1.75,0.37)--(1.75,1.37)--(0.95,0.85)--cycle;
    \fill[blue!07] (0,0)--(0.95,-0.15)--(0.95,0.85)--(0,1.00)--cycle;
    \draw[gray!55,densely dotted,thin] (0.80,0.52)--(0,0) (0.80,0.52)--(1.75,0.37)
                               (0.80,0.52)--(0.80,1.52);
    \draw[blue!60!black,thin]
      (0,0)--(0.95,-0.15)--(1.75,0.37) (0,0)--(0,1.00)
      (0.95,-0.15)--(0.95,0.85) (1.75,0.37)--(1.75,1.37)
      (0,1.00)--(0.95,0.85)--(1.75,1.37)--(0.80,1.52)--cycle;
    \foreach \p in {(0,0),(0.95,-0.15),(0.80,0.52),(0,1.00),(1.75,0.37),
                    (0.95,0.85),(0.80,1.52)}{\path[pv] \p circle (1.5pt);}
    \path[cut] (1.75,1.37) circle (2.1pt);
    \node[orange!75!black,right,inner sep=2.5pt,font=\tiny] at (1.78,1.47) {$(1,1,1)$};
  \end{scope}
  \node[sub] at (0.70,0.28) {$\mathcal{N}(f) = [0,1]^3$};
  \begin{scope}[shift={(2.95,1.50)}]
  \begin{scope}[hsplane]
  \fill[hs3] (0,0)--(0.700,-0.245)--(0.315,-0.560)--(-0.385,-0.315)--cycle;
  \fill[hs3] (0,0)--(0.700,-0.245)--(1.085,0.070)--(0.385,0.315)--cycle;
  \fill[hs3] (0,0)--(-0.700,0.245)--(-0.315,0.560)--(0.385,0.315)--cycle;
  \fill[hs3] (0,0)--(-0.700,0.245)--(-1.085,-0.070)--(-0.385,-0.315)--cycle;
  \fill[hs1] (0,0)--(-0.385,-0.315)--(-0.385,0.385)--(0,0.700)--cycle;
  \fill[hs1] (0,0)--(-0.385,-0.315)--(-0.385,-1.015)--(0,-0.700)--cycle;
  \fill[hs1] (0,0)--(0.385,0.315)--(0.385,-0.385)--(0,-0.700)--cycle;
  \fill[hs1] (0,0)--(0.385,0.315)--(0.385,1.015)--(0,0.700)--cycle;
  \fill[hs2] (0,0)--(0.700,-0.245)--(0.700,0.455)--(0,0.700)--cycle;
  \fill[hs2] (0,0)--(0.700,-0.245)--(0.700,-0.945)--(0,-0.700)--cycle;
  \fill[hs2] (0,0)--(-0.700,0.245)--(-0.700,-0.455)--(0,-0.700)--cycle;
  \fill[hs2] (0,0)--(-0.700,0.245)--(-0.700,0.945)--(0,0.700)--cycle;
  \end{scope}
  \draw[hsline] (1.085,0.070)--(0.315,-0.560)--(-1.085,-0.070)--(-0.315,0.560)--cycle;
  \draw[hsline] (-0.385,0.385)--(-0.385,-1.015)--(0.385,-0.385)--(0.385,1.015)--cycle;
  \draw[hsline] (0.700,0.455)--(0.700,-0.945)--(-0.700,-0.455)--(-0.700,0.945)--cycle;
  \draw[hsax] (0,0)--(-0.994,0.348); \draw[hsax] (0,0)--(0.547,0.447);
  \draw[hsax] (0,0)--(0,-0.994);
  \draw[hsax,densely dotted,->] (0,0)--(0.994,-0.348)
    node[font=\tiny,gray!55!black,below right,inner sep=1pt]{$x_1$};
  \draw[hsax,densely dotted,->] (0,0)--(-0.547,-0.447)
    node[font=\tiny,gray!55!black,below left,inner sep=1pt]{$x_2$};
  \draw[hsax,densely dotted,->] (0,0)--(0,0.994)
    node[font=\tiny,gray!55!black,above,inner sep=1pt]{$x_3$};
  \end{scope}
  \node[sub] at (2.95,0.06) {$\mathcal{T}(f)$: the three\\ planes $x_i{=}0$;\\ $2^3 = 8$ chambers};
  \begin{scope}[shift={(5.45,1.62)}]
    \fill[blue!14] (-1.10,-0.64)--(1.10,-0.64)--(0,1.28)--cycle;
    \draw[blue!60!black,thin] (-1.10,-0.64)--(1.10,-0.64)--(0,1.28)--cycle;
    \node[font=\tiny] at (0,0.00) {$\{1,2,3\}$};
  \end{scope}
  \node[sub] at (5.45,0.30) {the orthant $x<0$ seen\\ along the diagonal:\\ one chamber};
\end{scope}
\begin{scope}[xshift=7.10cm]
  \draw[cell] (-0.35,-0.80) rectangle (6.65,4.55);
  \node[align=center] at (3.15,4.25)
    {(R) \ $g = \min_{S \subsetneq [3]} x^{\mathbf 1_S}$ \ (proper subsets only)};
  \node[font=\tiny,gray!60!black,left] at (0.20,3.85) {$b{=}0$};
  \node[font=\tiny,gray!60!black,left] at (0.20,3.10) {$b{=}1$};
  \node[vtx,label=left:{$s$}] (Rv00) at (0.75,3.85) {};
  \node[vtx] (Rv10) at (1.80,3.85) {};
  \node[vtx] (Rv20) at (2.85,3.85) {};
  \node[vtx] (Rv11) at (1.80,3.10) {};
  \node[vtx] (Rv21) at (2.85,3.10) {};
  \node[vtx,label=right:{$z$}] (Rt) at (3.90,3.10) {};
  \foreach \u/\v/\lab in {Rv00/Rv10/{x_1},Rv10/Rv20/{x_2}}{
    \draw[ed] (\u) -- node[above,inner sep=1.5pt,font=\tiny]{$\lab$} (\v);}
  \foreach \u/\v in {Rv00/Rv11,Rv10/Rv21,Rv20/Rt}{
    \draw[ed] (\u) -- node[pos=0.30,right,inner sep=1pt,font=\tiny]{$0$} (\v);}
  \foreach \u/\v/\lab in {Rv11/Rv21/{x_2},Rv21/Rt/{x_3}}{
    \draw[ed] (\u) -- node[above,inner sep=1.5pt,font=\tiny]{$0$} (\v);
    \draw[ed] (\u) to[bend right=30] node[below,inner sep=1pt,font=\tiny]{$\lab$} (\v);}
  \begin{scope}[shift={(0.20,0.70)},scale=0.66]
    \fill[blue!10] (0,1.00)--(0.95,0.85)--(0.80,1.52)--cycle;
    \fill[blue!14] (0.95,-0.15)--(1.75,0.37)--(0.95,0.85)--cycle;
    \fill[blue!07] (0,0)--(0.95,-0.15)--(0.95,0.85)--(0,1.00)--cycle;
    \fill[orange!22] (1.75,0.37)--(0.95,0.85)--(0.80,1.52)--cycle;
    \draw[gray!55,densely dotted,thin] (0.80,0.52)--(0,0) (0.80,0.52)--(1.75,0.37)
                               (0.80,0.52)--(0.80,1.52);
    \draw[blue!60!black,thin]
      (0,0)--(0.95,-0.15)--(1.75,0.37) (0,0)--(0,1.00)
      (0.95,-0.15)--(0.95,0.85)
      (0,1.00)--(0.95,0.85) (0.80,1.52)--(0,1.00);
    \draw[orange!75!black,semithick]
      (1.75,0.37)--(0.95,0.85)--(0.80,1.52)--cycle;
    \draw[gray!60,dotted,thin]
      (1.75,0.37)--(1.75,1.37) (0.95,0.85)--(1.75,1.37) (0.80,1.52)--(1.75,1.37);
    \foreach \p in {(0,0),(0.95,-0.15),(0.80,0.52),(0,1.00),(1.75,0.37),
                    (0.95,0.85),(0.80,1.52)}{\path[pv] \p circle (1.5pt);}
    \draw[orange!75!black,densely dotted] (1.75,1.37) circle (2.6pt);
    \node[orange!75!black,right,inner sep=2.5pt,font=\tiny,align=left]
      at (1.80,1.50) {$(1,1,1)$\\ absent};
  \end{scope}
  \node[sub] at (0.70,0.28)
    {$\mathcal{N}(g)$: the cube\\ cut at $\sum_i \alpha_i \le 2$};
  \begin{scope}[shift={(2.95,1.50)}]
  \begin{scope}[hsplane]
  \fill[hs3] (0,0)--(0.700,-0.245)--(0.315,-0.560)--(-0.385,-0.315)--cycle;
  \fill[hs3] (0,0)--(0.700,-0.245)--(1.085,0.070)--(0.385,0.315)--cycle;
  \fill[hs3] (0,0)--(-0.700,0.245)--(-1.085,-0.070)--(-0.385,-0.315)--cycle;
  \fill[hs1] (0,0)--(-0.385,-0.315)--(-0.385,0.385)--(0,0.700)--cycle;
  \fill[hs1] (0,0)--(-0.385,-0.315)--(-0.385,-1.015)--(0,-0.700)--cycle;
  \fill[hs1] (0,0)--(0.385,0.315)--(0.385,1.015)--(0,0.700)--cycle;
  \fill[hs2] (0,0)--(0.700,-0.245)--(0.700,0.455)--(0,0.700)--cycle;
  \fill[hs2] (0,0)--(0.700,-0.245)--(0.700,-0.945)--(0,-0.700)--cycle;
  \fill[hs2] (0,0)--(-0.700,0.245)--(-0.700,0.945)--(0,0.700)--cycle;
  \end{scope}
  \draw[gray!55,densely dotted,thin]
    (0,0)--(-0.700,0.245)--(-0.315,0.560)--(0.385,0.315)--cycle
    (0,0)--(0.385,0.315)--(0.385,-0.385)--(0,-0.700)--cycle
    (0,0)--(-0.700,0.245)--(-0.700,-0.455)--(0,-0.700)--cycle;
  \fill[orange!45,opacity=.62] (0,0)--(-0.473,-0.210)--(0,-0.700)--cycle;
  \fill[orange!45,opacity=.62] (0,0)--(-0.473,-0.210)--(0.385,0.315)--cycle;
  \fill[orange!45,opacity=.62] (0,0)--(-0.473,-0.210)--(-0.700,0.245)--cycle;
  \draw[orange!78!black,thin]
    (-0.473,-0.210)--(0,-0.700) (-0.473,-0.210)--(0.385,0.315)
    (-0.473,-0.210)--(-0.700,0.245);
  \draw[orange!80!black,semithick] (0,0)--(-0.473,-0.210);
  \draw[hsline] (1.085,0.070)--(0.315,-0.560)--(-1.085,-0.070)--(-0.315,0.560)--cycle;
  \draw[hsline] (-0.385,0.385)--(-0.385,-1.015)--(0.385,-0.385)--(0.385,1.015)--cycle;
  \draw[hsline] (0.700,0.455)--(0.700,-0.945)--(-0.700,-0.455)--(-0.700,0.945)--cycle;
  \draw[hsax] (0,0)--(-0.994,0.348); \draw[hsax] (0,0)--(0.547,0.447);
  \draw[hsax] (0,0)--(0,-0.994);
  \draw[hsax,densely dotted,->] (0,0)--(0.994,-0.348)
    node[font=\tiny,gray!55!black,below right,inner sep=1pt]{$x_1$};
  \draw[hsax,densely dotted,->] (0,0)--(-0.547,-0.447)
    node[font=\tiny,gray!55!black,below left,inner sep=1pt]{$x_2$};
  \draw[hsax,densely dotted,->] (0,0)--(0,0.994)
    node[font=\tiny,gray!55!black,above,inner sep=1pt]{$x_3$};
  \end{scope}
  \node[sub] at (2.95,0.06) {$\mathcal{T}(g)$: dotted walls\\ replaced by $x_i{=}x_j$;\\ $7$ chambers};
  \begin{scope}[shift={(5.45,1.62)}]
    \fill[orange!16] (-1.10,-0.64)--(1.10,-0.64)--(0,0)--cycle;
    \fill[orange!26] (1.10,-0.64)--(0,1.28)--(0,0)--cycle;
    \fill[orange!36] (-1.10,-0.64)--(0,1.28)--(0,0)--cycle;
    \draw[blue!60!black,thin] (-1.10,-0.64)--(1.10,-0.64)--(0,1.28)--cycle;
    \draw[orange!80!black,semithick]
      (0,0)--(-1.10,-0.64) (0,0)--(1.10,-0.64) (0,0)--(0,1.28);
    \node[font=\tiny,scale=0.78] at (0,-0.43) {$\{1,2\}$};
    \node[font=\tiny,scale=0.78] at (-0.34,0.12) {$\{1,3\}$};
    \node[font=\tiny,scale=0.78] at ( 0.34,0.12) {$\{2,3\}$};
  \end{scope}
  \node[sub] at (5.45,0.30) {the orthant $x<0$ seen\\ along the diagonal: three\\ chambers, meeting at $x_i{=}x_j$};
\end{scope}
\begin{scope}[yshift=-0.80cm]
  \draw[cell,line width=0.9pt] (2.40,-2.10) rectangle (11.30,0);
  \node[align=left] at (4.05,-0.35) {observed: \ $\min_{S \subseteq [2]} x^{\mathbf 1_S}$};
  \node[vtx,label=left:{$s$}] (ds) at (3.25,-1.30) {};
  \node[vtx] (dm) at (4.15,-1.30) {};
  \node[vtx,label=right:{$z$}] (dt) at (5.05,-1.30) {};
  \foreach \u/\v/\lab in {ds/dm/{x_1},dm/dt/{x_2}}{
    \draw[ed] (\u) to[bend left=38] node[above,inner sep=1.5pt]{$0$} (\v);
    \draw[ed] (\u) to[bend right=38] node[below,inner sep=1.5pt]{$\lab$} (\v);}
  \begin{scope}[shift={(6.35,-1.62)},scale=0.85]
    \draw[->,gray!60] (0,0)--(1.30,0) node[right,black,inner sep=1.5pt,font=\tiny]{$\alpha_1$};
    \draw[->,gray!60] (0,0)--(0,1.30) node[above,black,inner sep=1.5pt,font=\tiny]{$\alpha_2$};
    \fill[blue!12] (0,0)--(0.95,0)--(0.95,0.95)--(0,0.95)--cycle;
    \draw[blue!60!black,thin] (0,0)--(0.95,0)--(0.95,0.95)--(0,0.95)--cycle;
    \foreach \p in {(0,0),(0.95,0),(0,0.95),(0.95,0.95)}{\path[pv] \p circle (1.5pt);}
  \end{scope}
  \node[sub] at (6.60,-1.90) {$\Gamma$: the unit square};
  \begin{scope}[shift={(9.55,-1.02)}]
    \fill[blue!6] (-0.72,-0.72) rectangle (0.72,0.72);
    \draw[gray!35,thin] (-0.72,-0.72) rectangle (0.72,0.72);
    \draw[blue!60!black,line width=0.9pt] (-0.72,0)--(0.72,0);
    \draw[blue!60!black,line width=0.9pt] (0,-0.72)--(0,0.72);
    \fill[black] (0,0) circle (1.3pt);
    \node[font=\tiny] at ( 0.38, 0.42) {$\emptyset$};
    \node[font=\tiny] at (-0.38, 0.42) {$\{1\}$};
    \node[font=\tiny] at ( 0.38,-0.42) {$\{2\}$};
    \node[font=\tiny] at (-0.38,-0.42) {$\{1,2\}$};
    \node[font=\tiny,gray!55!black,right,inner sep=2pt] at (0.72,0) {$x_1$};
    \node[font=\tiny,gray!55!black,above,inner sep=2pt] at (0,0.72) {$x_2$};
  \end{scope}
  \node[sub] at (9.65,-1.90) {$\mathcal{T}$: two lines, $4$ chambers};
\end{scope}
\draw[->,line width=0.9pt,gray!70] (1.10,-0.82) to[out=-90,in=180]
  node[pos=0.30,below left,black,inner sep=1.5pt]{$P_3,\,Q_3$} (2.40,-1.85);
\draw[->,line width=0.9pt,gray!70] (12.65,-0.82) to[out=-90,in=0]
  node[pos=0.30,below right,black,inner sep=1.5pt]{$P_3,\,Q_3$} (11.30,-1.85);
\end{tikzpicture}}
  \caption{Two length-$3$ problems that no substitution distinguishes.
  \textbf{(L)}~min-subset-sum, support $\{0,1\}^3$;
  \textbf{(R)}~the same over proper subsets, support
  $\{0,1\}^3\setminus\{(1,1,1)\}$.
  Middle panels show the decision boundaries (positive half-axes behind, dotted).
  The all-ones monomial is off the coordinate face and redundant under the projection,
  so both $P_3$ and $Q_3$, applied to either problem, yield the same length-$2$ object
  (bottom).}
  \label{fig:fiber-Pi-preimages}
\end{figure}

\section{A hierarchy beyond series and parallel}\label{sec:sp-limits}

Since substitution is insufficient, cross-length structure must come from
operations that \emph{build} a longer instance from shorter ones.
Consider Bellman's equation of optimal substructure,
$f_{t+1}(x) = \min_{d}\bigl(w_d + f_t^{(d)}(x)\bigr)$. Series and parallel composition correspond to two ways of modifying a Bellman expression at time $t$: extending the time step and updating the prior, respectively.
We show that these operations are doubly insufficient to express the topology of DP instances.

\paragraph{Series-parallel $\subsetneq$ terminal-only.}
A terminal-only operation may be thought of as an \emph{edge substitution} $\Phi_H$ on a template (unweighted) DAG $H$, replacing each edge with a sub-DAG (Appendix~\ref{app:sp-boundary}).
Series and parallel composition are the two simplest edge substitutions (templates with one
or two edges), but they do not generate all edge-substitution topologies: the witness is the
Wheatstone graph $W$ (Figure~\ref{fig:sp-limits}, left).

\begin{proposition}\label{prop:wheatstone-glue}
The Wheatstone edge substitution $\Phi_W$
(Figure~\ref{fig:sp-limits}) is not generated by series and parallel composition since $W$ is non-series-parallel
\citep{duffin1965-sp,valdes1982-sp-digraphs}.
\end{proposition}

\paragraph{Terminal-only $\subsetneq$ all DP compositions.} Clamped min-subarray
(Figure~\ref{fig:sp-limits}, right) composes two instances by gluing
$s_A{=}s_B$, $z_A{=}z_B$, \emph{and} the last internal vertex of~$A$ to the first
of~$B$, so that a subarray may cross the boundary --- an identification of \emph{internal} vertices.

\begin{figure}[htbp]
  \centering
\resizebox{\linewidth}{!}{%
\begin{tikzpicture}[>=Stealth,scale=1.0,
  every node/.style={font=\small,transform shape},
  v/.style={circle,draw,inner sep=0pt,minimum size=4.6mm},
  t/.style={v,fill=black!8},
  box/.style={draw,rounded corners,minimum width=8.5mm,minimum height=6.5mm},
  el/.style={midway,above=0.5pt,font=\scriptsize,fill=white,inner sep=1pt},
  tap/.style={->,black!45,shorten >=1pt},
  mg/.style={<->,dashed,line width=1pt,red!70!black,shorten >=2pt,shorten <=2pt},
  ml/.style={midway,font=\scriptsize,color=red!70!black,fill=white,inner sep=1.2pt}]
  \node[v] (sig) at (0,0) {};
  \node[v] (a)   at (2.1,1.05) {};
  \node[v] (b)   at (2.1,-1.05) {};
  \node[v] (tau) at (4.2,0) {};
  \node[box] (A1) at (1.05,0.72) {$A_1$};
  \node[box] (A2) at (1.05,-0.72) {$A_2$};
  \node[box] (A3) at (2.1,0) {$A_3$};
  \node[box] (A4) at (3.15,0.72) {$A_4$};
  \node[box] (A5) at (3.15,-0.72) {$A_5$};
  \draw[->] (sig)--(A1); \draw[->] (A1)--(a);
  \draw[->] (sig)--(A2); \draw[->] (A2)--(b);
  \draw[->] (a)--(A3);   \draw[->] (A3)--(b);
  \draw[->] (a)--(A4);   \draw[->] (A4)--(tau);
  \draw[->] (b)--(A5);   \draw[->] (A5)--(tau);
  \node at (2.1,-1.85) {$\Phi_W(A_1,\dots,A_5)$};
  \begin{scope}[xshift=5.4cm]
    \node[t] (sA) at (1.75,1.35) {};
    \node[t] (zA) at (1.75,-1.35) {};
    \node[v] (a0) at (0,0) {};
    \node[v] (a1) at (1.5,0) {};
    \node at (2.15,0) {$\cdots$};
    \node[v] (a2) at (3.5,0) {};
    \draw[->] (a0)--(a1) node[el]{$x_1^A$};
    \draw[->] (a1)--(1.85,0);
    \draw[->] (2.45,0)--(a2) node[el]{$x_{t_A-1}^A$};
    \foreach \nd in {a0,a1,a2}{\draw[tap] (sA)--(\nd); \draw[tap] (\nd)--(zA);}
    \node[font=\scriptsize,black!60] at (1.75,-1.9) {$A$: all taps weight $0$};
    \node[t] (sB) at (6.85,1.35) {};
    \node[t] (zB) at (6.85,-1.35) {};
    \node[v] (b0) at (5.1,0) {};
    \node[v] (b1) at (6.6,0) {};
    \node at (7.25,0) {$\cdots$};
    \node[v] (b2) at (8.6,0) {};
    \draw[->] (b0)--(b1) node[el]{$x_1^B$};
    \draw[->] (b1)--(6.95,0);
    \draw[->] (7.55,0)--(b2) node[el]{$x_{t_B-1}^B$};
    \foreach \nd in {b0,b1,b2}{\draw[tap] (sB)--(\nd); \draw[tap] (\nd)--(zB);}
    \node[font=\scriptsize,black!60] at (6.85,-1.9) {$B$: all taps weight $0$};
    \draw[mg] (sA) to[bend left=13]  node[ml]{merge} (sB);
    \draw[mg] (zA) to[bend right=13] node[ml]{merge} (zB);
    \draw[mg] (a2) to[bend right=32] node[ml,below=2pt]{merge} (b0);
  \end{scope}
\end{tikzpicture}}
  \caption{\textbf{(L)}~The Wheatstone edge substitution $\Phi_W(A_1,\dots,A_5)$.
  \textbf{(R)}~Merging two instances of clamped min-subarray: the two sources, the two sinks,
  and the last internal vertex of~$A$ with the first internal vertex of~$B$.}
  \label{fig:sp-limits}
\end{figure}
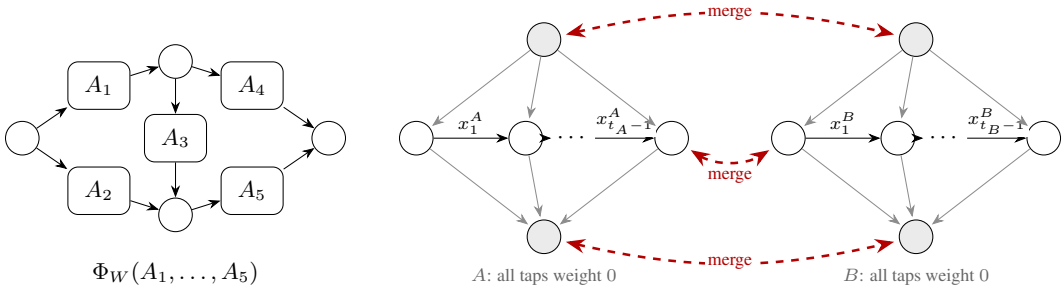

\section{Conclusion}\label{sec:conclusion}

The triality (Theorem~\ref{thm:triality}) means that compositional structure discovered in any one of the three languages --- graph, algebra, polyhedron --- transfers exactly to the others. The invariant calculus then characterizes the full redundancy in the DAG encoding, identifying which structural differences are artifacts of the representation rather than intrinsic properties of the computed function.
The negative results constrain equivariant architecture design. Semiring-native substitutions cannot lift the polyhedral geometry from length $T$ to length $T{+}1$ (Section~\ref{sec:geometry}), and the compositional symmetries available to Bellman-style reasoning are doubly incomplete: series-parallel does not generate all DAG topologies, and terminal-only operations do not capture all DP compositions (Section~\ref{sec:sp-limits}). Together these bound what $(\min, +)$ structure alone in an architecture can achieve for length generalization, and point toward richer compositional operations that would be needed.

\begin{ack}
This work was partially supported by the NSF Division of Mathematical Sciences: Statistics Program DMS 2409819, and by the Peter Buck Student Internship Fund administered by Bowdoin College. The authors thank Baran Hashemi (Postdoctoral Scholar, MPI MiS Leipzig) for his generous feedback prior to submission.
\end{ack}

\bibliographystyle{plainnat}
\bibliography{ref}

\clearpage
\appendix

\section{Series and parallel composition}\label{app:sp-comp}
\begin{figure}[htbp]
  \centering
\begin{tikzpicture}[>=Stealth,scale=0.85,
  every node/.style={font=\small,transform shape},
  v/.style={circle,draw,inner sep=0pt,minimum size=5mm},
  box/.style={draw,rounded corners,minimum width=11mm,minimum height=8mm}]
  \node[v] (sA) at (0,0) {$s_A$};
  \node[box] (A) at (1.5,0) {$A$};
  \node[v] (mid) at (3.0,0) {$z_A{=}s_B$};
  \node[box] (B) at (4.9,0) {$B$};
  \node[v] (zB) at (6.4,0) {$z_B$};
  \draw[->] (sA)--(A); \draw[->] (A)--(mid); \draw[->] (mid)--(B); \draw[->] (B)--(zB);
  \node at (3.2,-1.20) {\small series $A;B$: $\ \dist_{A;B}=\dist_A + \dist_B$};
  \begin{scope}[xshift=7.7cm]
    \node[v] (s) at (0,0) {$s$};
    \node[box] (Ap) at (1.7,0.56) {$A$};
    \node[box] (Bp) at (1.7,-0.56) {$B$};
    \node[v] (z) at (3.4,0) {$z$};
    \draw[->] (s)--(Ap); \draw[->] (Ap)--(z);
    \draw[->] (s)--(Bp); \draw[->] (Bp)--(z);
    \node at (1.7,-1.20) {\small parallel $A\|B$: $\ \dist_{A\|B}=\min(\dist_A, \dist_B)$};
  \end{scope}
\end{tikzpicture}
  \caption{Series and parallel composition of two-terminal DAGs: series identifies $z_A$ with
  $s_B$, parallel identifies the two sources and the two sinks.}
  \label{fig:sp-comp}
\end{figure}
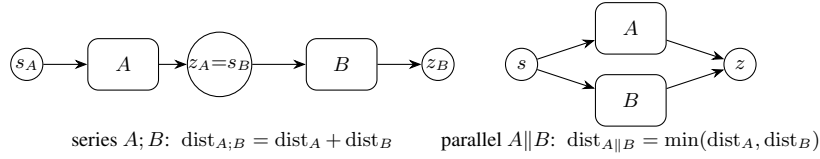

\section{Tropical geometry in a nutshell}\label{app:crash-course}

This appendix collects the definitions and examples needed to read the main text.
We follow the conventions of Joswig~\citep{joswig2021essentials} and
Maclagan--Sturmfels~\citep{maclagan-sturmfels2015}; proofs and formal statements
are in Appendix~\ref{app:scaffolding}.

\subsection{The tropical semiring}\label{app:cc-semiring}

The \emph{tropical semiring} is the set $\mathbb{T} = \mathbb{R} \cup \{+\infty\}$ equipped with
the operations: \emph{tropical addition} $a \oplus b \coloneqq \min(a,b)$ and \emph{tropical
multiplication} $a \odot b \coloneqq a + b$.  The additive identity is $+\infty$
($\min(a, {+}\infty) = a$) and the multiplicative identity is $0$ ($a + 0 = a$).  Every finite
element has a multiplicative inverse ($a \odot (-a) = 0$), but tropical addition is
\emph{idempotent}: $a \oplus a = a$, so no finite element has an additive inverse and
$\mathbb{T}$ is a semiring, not a ring.  The $(\max, +)$ convention is equivalent under
negation; we use $(\min, +)$ throughout, aligned with shortest-path formulations.

\subsection{Tropical polynomials}\label{app:cc-polynomials}

A \emph{tropical monomial} $c \odot x_1^{\odot\alpha_1} \odot \cdots \odot x_d^{\odot\alpha_d}$
(with $\alpha \in \mathbb{Z}_{\ge 0}^d$, $c \in \mathbb{R}$) evaluates to the affine
function $c + \alpha_1 x_1 + \cdots + \alpha_d x_d = c + \langle \alpha, x\rangle$, since
tropical exponentiation is repeated addition: $x^{\odot k} = kx$.

A \emph{tropical polynomial} is a finite tropical sum of monomials:
\[
  f(x) = \bigoplus_{\alpha \in S} c_\alpha \odot x^{\odot\alpha}
       = \min_\alpha\bigl(c_\alpha + \langle\alpha, x\rangle\bigr).
\]
Its graph is the lower envelope of finitely many affine functions --- piecewise linear and
concave.

\begin{example}[A univariate tropical polynomial]\label{eg:cc-univariate}
Let $f(x) = 4 \oplus (2 \odot x) \oplus (1 \odot x^{\odot 2}) \oplus (3 \odot x^{\odot 3})
= \min(4,\; 2{+}x,\; 1{+}2x,\; 3{+}3x)$.  Its graph is the lower envelope of four lines
(Figure~\ref{fig:cc-dome}).  The envelope bends at $x = -2, 1, 2$; this is the tropical
vanishing set (the points where two or more monomials tie for the minimum).
\end{example}

\paragraph{Formal vs.\ function.}
Two formal polynomials can evaluate to the same function: for instance,
$\min(0, x, 2x) = \min(0, 2x)$ for all $x$, because the monomial $x$ is never the
unique minimizer.  This formal/function distinction is what the two rows of the triality
(Theorem~\ref{thm:triality}) capture.

\paragraph{Support and lifted support.}
The \emph{support} $S(f) = \{\alpha : c_\alpha \ne {+}\infty\}$ records which monomials are
present.  The \emph{lifted support}
$\hat S(f) = \{(\alpha, c_\alpha) : \alpha \in S(f)\} \subset \mathbb{R}^d \times \mathbb{R}$
raises each exponent to its coefficient height.  When two terms share an exponent, their
tropical sum keeps the lesser coefficient; this per-exponent reduction is called
$\vmin$ (vertical minimum), and a polynomial in \emph{normal form} has been reduced by
$\vmin$.

\subsection{The geometric dictionary}\label{app:cc-geometry}

\paragraph{The dome.}
The region weakly below the graph,
$\mathcal{D}(f) = \{(x,s) \in \mathbb{R}^d \times \mathbb{R} : s \le f(x)\}$, is a convex
polyhedron cut out by one half-space per monomial.  Each facet of $\mathcal{D}(f)$ is the
locus where one monomial is the unique minimizer; the function bends exactly at the ridges
(codimension-$2$ faces) of $\mathcal{D}(f)$.

\begin{figure}[htbp]
  \centering
\begin{tikzpicture}[x=0.72cm,y=0.36cm]
  \fill[blue!10]
    (-3.3,-6.9) -- (-2,-3) -- (1,3) -- (2,4) -- (3.6,4)
    -- (3.6,-7.7) -- (-3.3,-7.7) -- cycle;
  \draw[gray!60] (-3.45,-7.35) -- (0.85,5.55)
    node[above,font=\tiny] {$3{+}3x$};
  \draw[gray!60] (-3.45,-5.90) -- (2.25,5.50)
    node[above,font=\tiny] {$1{+}2x$};
  \draw[gray!60] (-3.45,-1.45) -- (3.45,5.45)
    node[above,font=\tiny] {$2{+}x$};
  \draw[gray!60] (-3.45,4) -- (3.75,4)
    node[right,font=\tiny] {$4$};
  \draw[->] (-3.6,0) -- (4.0,0) node[right,font=\scriptsize] {$x$};
  \draw[->] (0,-8.0) -- (0,6.0) node[above,font=\scriptsize] {$s$};
  \draw[very thick,red!80!black]
    (-3.3,-6.9) -- (-2,-3) -- (1,3) -- (2,4) -- (3.6,4);
  \foreach \x/\y in {-2/-3, 1/3, 2/4}{
    \fill[red!80!black] (\x,\y) circle (1.8pt);
    \draw[densely dashed,gray!60] (\x,\y) -- (\x,0);
  }
  \node[blue!55!black,font=\scriptsize] at (-1.0,-5.4) {$\mathcal{D}(f)$};
\end{tikzpicture}
  \caption{The dome $\mathcal{D}(f)$ of $f = \min(4,\; 2{+}x,\; 1{+}2x,\; 3{+}3x)$
  from Example~\ref{eg:cc-univariate}.  Each facet is a monomial's locus of dominance; the
  three vertices of the dome project to the tropical zeroes $\{-2, 1, 2\}$.}
  \label{fig:cc-dome}
\end{figure}
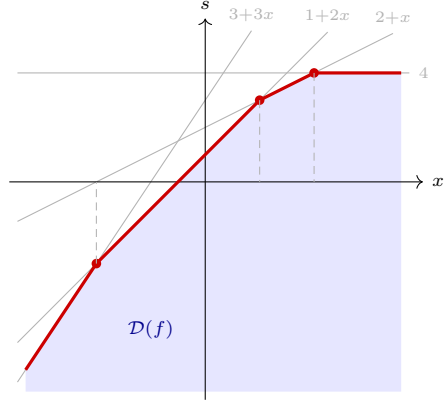

\paragraph{Newton polytope and extended Newton polyhedron.}
The \emph{Newton polytope} $\mathcal{N}(f) = \operatorname{conv}(S(f)) \subseteq \mathbb{R}^d$
records which monomials are present, forgetting coefficients.  The \emph{extended Newton
polyhedron}
\[
  \Gamma(f) = \operatorname{conv}\bigl(\hat S(f)\bigr) + \mathbb{R}_{\ge 0}\,e_{d+1}
\]
retains the coefficients: it is the convex hull of the lifted support, extended upward by an
infinite ray.  Its lower faces encode the polynomial's evaluation, and the key fact is that
$\Gamma(f) = \Gamma(g)$ if and only if $f$ and $g$ agree as functions
\citep[Thm.~1.13]{joswig2021essentials}.

\paragraph{Tropical hypersurface.}
The \emph{tropical hypersurface}
\[\mathcal{T}(f) = \{x \in \mathbb{R}^d : \text{the minimum in
} f(x) \text{ is attained by } {\ge}\,2 \text{ monomials}\}\]
is the locus where strategies tie
for optimality, i.e.\ a decision boundary.  It is a polyhedral complex of pure dimension $d - 1$, the projection of the
codimension-$2$ skeleton of the dome.

\begin{example}[The tropical line]\label{eg:cc-tropical-line}
The polynomial $f(x,y) = 0 \oplus x \oplus y = \min(0, x, y)$ has three monomials.  The
tropical hypersurface $\mathcal{T}(f)$ consists of three rays emanating from the origin
(Figure~\ref{fig:cc-tropical-line}):
$\{x = 0 \le y\}$, $\{y = 0 \le x\}$, and $\{x = y \le 0\}$.  Each ray separates two regions
of linearity, one per monomial that dominates there.
\end{example}

\begin{figure}[htbp]
  \centering
\begin{tikzpicture}[>=Stealth,scale=0.65]
  \fill[blue!6] (-2.2,-2.2) rectangle (2.2,2.2);
  \draw[very thick,red!80!black] (0,0) -- (2.2,0)
    node[right,font=\scriptsize,black] {$x{=}0$};
  \draw[very thick,red!80!black] (0,0) -- (0,2.2)
    node[above,font=\scriptsize,black] {$y{=}0$};
  \draw[very thick,red!80!black] (0,0) -- (-2.0,-2.0)
    node[below left,font=\scriptsize,black] {$x{=}y$};
  \fill[red!80!black] (0,0) circle (2pt);
  \node[font=\scriptsize] at (1.2,1.2) {$s{=}0$};
  \node[font=\scriptsize] at (-1.3,0.7) {$s{=}x$};
  \node[font=\scriptsize] at (0.7,-1.3) {$s{=}y$};
\end{tikzpicture}
  \caption{The tropical line of Example~\ref{eg:cc-tropical-line}: the hypersurface of
  $f = \min(0, x, y)$ is three rays meeting at the origin, separating the three regions on
  which a single monomial attains the minimum.}
  \label{fig:cc-tropical-line}
\end{figure}
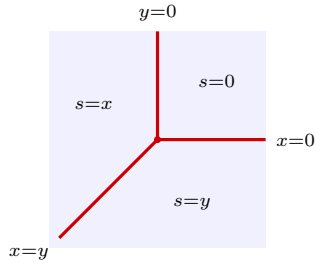

\subsection{Two-terminal DAGs and path polynomials}\label{app:cc-dags}

A \emph{two-terminal DAG} is an acyclic directed graph with a designated source $s$ and
sink $z$ such that every vertex lies on some $s$--$z$ path.  Its \emph{path polynomial} is
\[
  f_G(x) = \bigoplus_{\pi\,:\, s \to z}\; \bigodot_{e \in \pi} w_e
         = \min_{\pi\,:\, s \to z} \sum_{e \in \pi} w_e\,,
\]
the tropical sum over all $s$--$z$ paths of the tropical product of edge weights.  Each path
contributes one monomial; the polynomial is their tropical sum.

Series composition ($A ; B$, identifying $z_A = s_B$) gives
$f_{A;B} = f_A \odot f_B = f_A + f_B$; parallel composition ($A \| B$, identifying sources and
sinks) gives $f_{A\|B} = f_A \oplus f_B = \min(f_A, f_B)$.  Thus $G \mapsto f_G$ preserves both $\|$ and $;$ from two-terminal DAGs to tropical polynomials --- the starting point of the
triality in Section~\ref{sec:triality}.

\section{Polynomials, configurations, and polyhedra}\label{app:scaffolding}

\begin{definition}\label{def:lifted-config}
For a formal tropical polynomial $f = \min_\alpha (c_\alpha + x^\alpha)$ in $d$ variables
in normal form (one least coefficient per exponent), the \emph{support} is
$S(f) = \{\alpha : c_\alpha \ne +\infty\}$ and the \emph{lifted configuration} is
$\hat S(f) = \{(\alpha, c_\alpha) : \alpha \in S(f)\} \subseteq \mathbb{R}^d \times \mathbb{R}$.
\end{definition}

\begin{definition}\label{def:vmin}
For finite $S \subseteq \mathbb{Z}^d_{\ge 0} \times \mathbb{R}$, $\vmin(S)$ keeps over each exponent occurring
in $S$ only the least point above it; thus $\vmin(S) = S$ iff $S$ has at most one point over
each exponent.
\end{definition}

\begin{proposition}\label{prop:config-semiring}
Let $\mathcal{L}_d$ be the finite subsets of $\mathbb{Z}^d_{\ge 0} \times \mathbb{R}$ with at most one point over each
exponent, with addition $\vmin(S \cup T)$ and multiplication $\vmin\{(\alpha+\alpha',
c_\alpha + c_{\alpha'})\}$. Then $\mathcal{L}_d$ is a commutative idempotent semiring (identities
$\emptyset$ and $\{(0,0)\}$), and $f \mapsto \hat S(f)$ is a semiring isomorphism from
normal-form formal tropical polynomials onto $\mathcal{L}_d$.
\end{proposition}
\begin{proof}
Formal tropical Laurent polynomials are a semiring
\citep[Observation~2.15]{joswig2021essentials}; nonnegative-exponent configurations form a
sub-semiring containing $\emptyset$ and $\{(0,0)\}$, and $f \mapsto \hat S(f)$ is a relabelling
carrying tropical $\min$ to $\vmin$ of the union and $+$ to $\vmin$ of the sumset.
\end{proof}

\begin{definition}\label{def:ext-newt}
For $f$ with support $S$, $\Gamma(f) = \operatorname{conv}\{(\alpha, c_\alpha) : \alpha \in S\} +
(\{0\} \times \mathbb{R}_{\ge 0}) \subseteq \mathbb{R}^d \times \mathbb{R}$. Its \emph{lower faces} (outer normal
$(y,1)$) project onto $\mathcal{N}(f) = \operatorname{conv}(S)$ \citep[\S1.1]{joswig2021essentials}.
\end{definition}

\begin{lemma}\label{lem:gamma-function}
(i) $\Gamma(f)$ is its lower hull plus the upward ray, so the lower hull determines it. (ii)
$\Gamma(f) \mapsto f$ is a bijection from extended Newton polyhedra onto tropical polynomial
functions. \citep[Observation~1.8, Theorem~1.13]{joswig2021essentials}
\end{lemma}

\begin{lemma}\label{lem:gamma-quotient}
On $\mathcal{L}_d$ set $S \sim T$ iff $\operatorname{conv}(S) + \mathbb{R}_{\ge 0}e_{d+1} = \operatorname{conv}(T) + \mathbb{R}_{\ge 0}e_{d+1}$.
Then $\sim$ is a congruence for $(\min, +)$, and $\gamma : S \mapsto \operatorname{conv}(S) +
\mathbb{R}_{\ge 0}e_{d+1}$ induces a semiring isomorphism $\mathcal{L}_d/\!\sim\ \xrightarrow{\cong}
(\{\Gamma\}, \operatorname{conv}(\cdot \cup \cdot), +)$.
\end{lemma}
\renewcommand{\proofname}{Proof of Lemma~\ref{lem:gamma-quotient}}
\begin{proof}
Write $R = \mathbb{R}_{\ge 0}e_{d+1}$. Each $\Gamma$ satisfies $\Gamma + R = \Gamma$, hence is
determined by its lower support function --- the map
$x \mapsto \min_{(\alpha,c) \in \Gamma} (\langle \alpha, x\rangle + c)$, which is exactly the
tropical polynomial function carried by the configuration. The lower support function of a convex hull of a union is the pointwise
minimum of the two \citep[Cor.~16.5.1]{rockafellar1970convex}, i.e.\ $\min$; the support
function of a Minkowski sum is the sum \citep[Thm.~13.3]{rockafellar1970convex}, i.e.\ $+$.
A point discarded by $\vmin$ lies on $R$ above a retained point and is invisible to $\gamma$, so
$\gamma$ intertwines the operations; that this holds forces $\sim$ to be a congruence.
$\gamma(\emptyset) = \emptyset$, $\gamma(\{(0,0)\}) = R$ are the identities.
\end{proof}
\renewcommand{\proofname}{Proof}

\begin{lemma}\label{lem:descent}
Let $\varphi : R \to R'$ be a semiring isomorphism and $\theta, \theta'$ congruences with
$a \mathbin{\theta} b \iff \varphi(a) \mathbin{\theta'} \varphi(b)$. Then there is a unique
$\bar\varphi : R/\theta \to R'/\theta'$ with $\bar\varphi \circ q = q' \circ \varphi$, and it is
a semiring isomorphism.
\end{lemma}
\begin{proof}
We must show $\bar\varphi([a]) \coloneqq [\varphi(a)]_{\theta'}$ is a well-defined semiring
isomorphism $R/\theta \to R'/\theta'$.

\emph{Well-defined.}
If $[a]_\theta = [b]_\theta$, then $a \mathbin{\theta} b$, so by hypothesis
$\varphi(a) \mathbin{\theta'} \varphi(b)$, hence
$[\varphi(a)]_{\theta'} = [\varphi(b)]_{\theta'}$.

\emph{Injective.}
If $\bar\varphi([a]) = \bar\varphi([b])$, i.e.\
$\varphi(a) \mathbin{\theta'} \varphi(b)$, the reverse implication gives
$a \mathbin{\theta} b$, so $[a] = [b]$.

\emph{Surjective.}
For any $[r'] \in R'/\theta'$, surjectivity of $\varphi$ gives $a$ with
$\varphi(a) = r'$, so $\bar\varphi([a]) = [r']$.

\emph{Homomorphism.}
For either operation $*$ (with $*'$ its counterpart in $R'$):
\[
  \bar\varphi([a] * [b])
  = \bar\varphi([a * b])
  = [\varphi(a * b)]_{\theta'}
  = [\varphi(a) *' \varphi(b)]_{\theta'}
  = [\varphi(a)]_{\theta'} *' [\varphi(b)]_{\theta'}
  = \bar\varphi([a]) *' \bar\varphi([b]),
\]
using that $\varphi$ is a homomorphism (third equality) and that $\theta'$ is a
congruence (fourth equality). The identity elements map correctly since $\varphi$
preserves them.

\emph{Uniqueness.} Any $\psi$ with $\psi \circ q = q' \circ \varphi$ satisfies
$\psi([a]) = q'(\varphi(a)) = [\varphi(a)]_{\theta'} = \bar\varphi([a])$.
\end{proof}

\begin{proposition}\label{prop:path-hom}
For two-terminal DAGs, $f_{A;B} = f_A + f_B$ and $f_{A\|B} = \min(f_A, f_B)$; hence the two
congruences $\approx_{\mathrm f}, \approx_{\mathrm v}$ are compatible with $;$ and $\|$.
\end{proposition}
\begin{proof}
The $s$--$z$ paths of $A;B$ are concatenations of an $A$-path and a $B$-path, weight the
sum, so distributivity gives $f_A + f_B$; those of $A \| B$ are the disjoint
union, so the $\min$ splits as $\min(f_A, f_B)$. Normal form commutes with both.
\end{proof}

\section{\texorpdfstring{Proof of the triality (Theorem~\ref{thm:triality})}{Proof of the triality}}\label{app:triality}
\renewcommand{\proofname}{Proof}
\begin{proof}
Write $\Phi : [G] \mapsto f_G$.

\emph{(1) Formal row.} Since $\approx_{\mathrm f}$ is a congruence
(Proposition~\ref{prop:path-hom}), $\;\|\;$ and $;$ are well-defined on classes.
\begin{itemize}
\item \emph{Injective:} by the definition of $\approx_{\mathrm f}$.
\item \emph{Operation-preserving:} by Proposition~\ref{prop:path-hom}.
\item \emph{Identities:} the single $+\infty$-edge maps to $\emptyset$, the single $0$-edge to
  $\{(0,0)\}$.
\item \emph{Surjective:} given normal-form $f$, build one chain per support monomial ---
  $\alpha_i$ edges labelled $x_i$ and one edge of weight $c_\alpha$ --- and place the chains in
  parallel. Its path polynomial is $f$.
\end{itemize}
So $\Phi$ is an isomorphism onto the formal polynomials, a semiring by
Proposition~\ref{prop:config-semiring}; composing with $f \mapsto \hat S(f)$ gives the chain.

\emph{(2, 3) Function row.} Two graphs compute the same function iff their polynomials do iff
their extended Newton polyhedra agree, so the three congruences match under the formal-row
isomorphisms and the descent lemma applies. Concretely, let
$\theta_{\mathrm G}, \theta_{\mathrm F}, \theta_{\mathcal{L}}$ be, respectively:
\begin{itemize}
\item $[G]_{\mathrm f} \mathbin{\theta_{\mathrm G}} [H]_{\mathrm f}$ iff
  $G \approx_{\mathrm v} H$;
\item $f \mathbin{\theta_{\mathrm F}} g$ iff they evaluate to the same function;
\item $\sim$ of Lemma~\ref{lem:gamma-quotient}.
\end{itemize}
Now $G \approx_{\mathrm v} H$ iff $f_G, f_H$ evaluate equally, so $\theta_{\mathrm G}$ and
$\theta_{\mathrm F}$ correspond under $[G] \mapsto f_G$; and
$f \mathbin{\theta_{\mathrm F}} g$ iff $\Gamma(f) = \Gamma(g)$
(Lemma~\ref{lem:gamma-function}(ii)) iff $\hat S(f) \sim \hat S(g)$, so $\theta_{\mathrm F}$
and $\theta_{\mathcal{L}}$ correspond. $\theta_{\mathcal{L}}$ is a congruence
(Lemma~\ref{lem:gamma-quotient}), hence so are the others by transport. Applying
Lemma~\ref{lem:descent} to the two isomorphisms of~(1) with these matched congruences yields the
function-row isomorphisms, and since the vertical maps are the quotient maps, both squares
commute. No operation on $\Gamma$ is recomputed.
\end{proof}

\section{Invariance: normal form and completeness}\label{app:invariance}

\begin{definition}\label{def:deshared}
A two-terminal DAG is \emph{deshared} if distinct $s$--$z$ paths meet only at $s$ and $z$ (a
parallel bundle of chains). $D_{\mathrm f}(G)$ has one canonicalized chain per support monomial
of $f_G$; $D_{\mathrm v}(G)$ deletes additionally every chain whose lifted point lies weakly
above the lower hull of the others. Both may be constructed directly from the path polynomial itself, so DAGs $G, G'$ are equal exactly when their path polynomials are equal.
\end{definition}

\begin{lemma}\label{lem:deshare-count}
For a two-terminal DAG $G$, with $P(v)$ the number of $s$--$z$ paths through internal vertex $v$,
put $\Pi(G) = \sum_v (P(v) - 1)$. Then (1) $G$ is deshared iff $\Pi(G) = 0$; (2) a vertex split
is available iff $\Pi(G) > 0$; (3) exactly $\Pi(G)$ splits carry $G$ to deshared form.
\end{lemma}
\begin{proof}
By two-terminality each internal $v$ has $P(v) \ge 1$.
\begin{itemize}
\item[(1)] $\Pi(G) = 0$ iff every $P(v) = 1$ iff every internal vertex has in- and out-degree
  $1$ on the path subgraph, which is desharedness.
\item[(2)] $\Pi(G) > 0$ iff some vertex has in- or out-degree $> 1$, which admits a split.
\item[(3)] A split of $v$ into $v_1, v_2$ has $P(v_1) + P(v_2) = P(v)$ and changes no other
  $P(v')$, so
  \[
    \Pi \text{ drops by } (P(v)-1) - (P(v_1)-1) - (P(v_2)-1) = 1;
  \]
  with~(2), exactly $\Pi(G)$ splits reach deshared form.
\end{itemize}
\end{proof}

\begin{proposition}\label{prop:list-inv-f}
The following preserve $\approx_{\mathrm f}$: (1) \emph{series-run reweighting}
--- permuting and contracting the edges within a chain, preserving its sum; (2)
\emph{parallel-bundle bookkeeping} --- reordering parallel branches, combining two with equal
exponent by $\min$ of coefficients (idempotency when equal), and deleting or adjoining a branch
dominated at its own exponent; (3) \emph{tropical-zero adjunction} --- adding or removing a parallel
$+\infty$-edge or a subdividing $0$-edge; (4) \emph{vertex splitting and merging}
(Figure~\ref{fig:vertex-split}).
\end{proposition}
\begin{proof}
Write $\mathcal{P}(G)$ for the multiset of $s$--$z$ path monomials. Since $f_G$ is the $\min$
over $\mathcal{P}(G)$ in normal form, it suffices to show each operation leaves $\mathcal{P}(G)$
unchanged up to dominated entries.
\begin{itemize}
\item[(1)] The run is unbranched, so every path through it picks up the $+$-product of the
  run's weights; paths avoiding it are untouched.
\item[(2)] Permuting branches permutes $\mathcal{P}(G)$. Deletion: the surviving branch $A$ at
  the same exponent with $c_A \le c_C$ already attains the $\min$-least coefficient there.
\item[(3)] A path through a $+\infty$-edge has constant $+\infty$ and is dominated or
  unsupported; deletion reads the argument backwards.
\item[(4)] Out-split at $w$ along $\mathrm{out}(w) = O_1 \sqcup O_2$; in-split is symmetric.
  \emph{Acyclicity:} a cycle of $G'$ through $w_j$ maps to a cycle of $G$ through $w$.
  \emph{Path bijection:} an $s$--$z$ path of $G$ meeting $w$, entering along
  $e \in \mathrm{in}(w)$ and leaving along $e' \in O_j$, lifts to an $s$--$z$ path of $G'$
  through $w_j$ of the same weight; conversely, a path of $G'$ meets at most one of
  $w_1, w_2$ (else a cycle exists), and identifying them recovers a path of $G$. So
  $\mathcal{P}(G') = \mathcal{P}(G)$. Merging is the inverse.
\end{itemize}
\end{proof}

\begin{figure}[htbp]
  \centering
\resizebox{\linewidth}{!}{%
\begin{tikzpicture}[>=Stealth,every node/.style={font=\small},
  v/.style={circle,draw,inner sep=0pt,minimum size=5mm}]
  \node[v] (p1) at (0,0.9) {$p_1$};
  \node[v] (p2) at (0,-0.9) {$p_2$};
  \node[v] (w) at (1.6,0) {$w$};
  \node[v] (q1) at (3.2,0.9) {$q_1$};
  \node[v] (q2) at (3.2,-0.9) {$q_2$};
  \draw[->] (p1)--node[above]{$e_1$}(w);
  \draw[->] (p2)--node[below]{$e_2$}(w);
  \draw[->] (w)--node[above]{$g_1$}(q1);
  \draw[->] (w)--node[below]{$g_2$}(q2);
  \node at (1.6,-2.6) {\small original};
  \begin{scope}[xshift=5.4cm]
    \node[v] (p1) at (0,0.9) {$p_1$};
    \node[v] (p2) at (0,-0.9) {$p_2$};
    \node[v] (w1) at (1.9,1.4) {$w_1$};
    \node[v] (w2) at (1.9,-1.4) {$w_2$};
    \node[v] (q1) at (3.8,1.4) {$q_1$};
    \node[v] (q2) at (3.8,-1.4) {$q_2$};
    \draw[->] (p1)--node[above]{$e_1$}(w1);
    \draw[->] (p2)--node[above,pos=0.72,inner sep=1.5pt]{$e_2$}(w1);
    \draw[->] (p1)--node[below,pos=0.72,inner sep=1.5pt]{$e_1$}(w2);
    \draw[->] (p2)--node[below]{$e_2$}(w2);
    \draw[->] (w1)--node[above]{$g_1$}(q1);
    \draw[->] (w2)--node[below]{$g_2$}(q2);
    \node at (1.9,-2.6) {\small out-split, $O_1=\{g_1\}$, $O_2=\{g_2\}$};
  \end{scope}
  \begin{scope}[xshift=11.4cm]
    \node[v] (p1) at (0,1.4) {$p_1$};
    \node[v] (p2) at (0,-1.4) {$p_2$};
    \node[v] (w1) at (1.9,1.4) {$w_1$};
    \node[v] (w2) at (1.9,-1.4) {$w_2$};
    \node[v] (q1) at (3.8,0.9) {$q_1$};
    \node[v] (q2) at (3.8,-0.9) {$q_2$};
    \draw[->] (p1)--node[above]{$e_1$}(w1);
    \draw[->] (p2)--node[below]{$e_2$}(w2);
    \draw[->] (w1)--node[above]{$g_1$}(q1);
    \draw[->] (w1)--node[above,pos=0.28,inner sep=1.5pt]{$g_2$}(q2);
    \draw[->] (w2)--node[below,pos=0.28,inner sep=1.5pt]{$g_1$}(q1);
    \draw[->] (w2)--node[below]{$g_2$}(q2);
    \node at (1.9,-2.6) {\small in-split, $I_1=\{e_1\}$, $I_2=\{e_2\}$};
  \end{scope}
\end{tikzpicture}}
  \caption{The two vertex splits, realizing the distributivity clause of
  Proposition~\ref{prop:list-inv-f}. Left: an internal vertex $w$ with in-edges
  $e_1, e_2$ and out-edges $g_1, g_2$. Middle: the out-split along $O_1 = \{g_1\}$,
  $O_2 = \{g_2\}$, with each of $w_1, w_2$ carrying its own copy of $e_1$ and $e_2$.
  Right: the in-split along $I_1 = \{e_1\}$, $I_2 = \{e_2\}$, with each of $w_1, w_2$
  carrying its own copy of $g_1$ and $g_2$. In both panels every $p_i \to q_j$ path of
  the original survives exactly once.}
  \label{fig:vertex-split}
\end{figure}
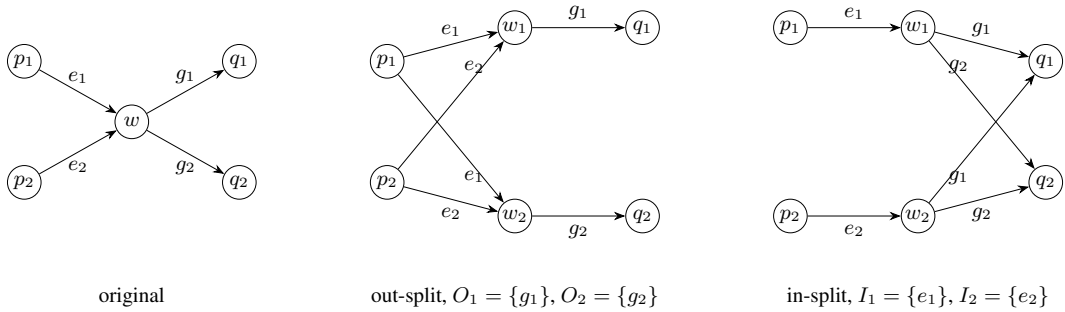

\begin{lemma}\label{lem:mediant}
On a bundle of chains with pairwise distinct exponents and no $+\infty$-edge, deleting a chain
preserves the computed function iff its lifted point $(\alpha, c_\alpha)$ lies weakly above the
lower hull of the others, iff it is not a vertex of $\Gamma$, iff it is nowhere the unique
minimizer (Figure~\ref{fig:mediant}); and at most $d+1$ chains ($d = \dim \mathcal{N}(f)$) are
needed. The forward direction (a certificate $\Rightarrow$ deletion preserves the function) is
unconditional.
\end{lemma}
\begin{proof}
Write $f_C(x) = \langle \alpha_C, x \rangle + c_C$ and
$f_{A_i}(x) = \langle \alpha_i, x \rangle + c_i$; set
$S' = \{(\alpha_i, c_i)\}_{i \le r}$.

\emph{Certificate $\Rightarrow$ deletion (unconditional).} A certificate gives
$N\alpha_C = \sum_i k_i \alpha_i$ and $N c_C \ge \sum_i k_i c_i$, so for every $x$,
\[
  N f_C(x)
  = \langle N\alpha_C, x\rangle + N c_C
  \ge \sum_i k_i f_{A_i}(x)
  \ge N \min_j f_{A_j}(x).
\]
Hence $f_C \ge \min_j f_{A_j}$ pointwise, and a branch that never strictly beats the others may
be dropped from a $\min$ without changing it.

\emph{Equivalences (under distinctness and no $+\infty$).}
\begin{itemize}
\item \emph{Not a vertex $\Leftrightarrow$ nowhere the unique minimizer.}
  $\Gamma(f_B)$ recedes in the $e_{d+1}$ direction, so every supporting functional with positive
  last coordinate scales to $(\alpha, c) \mapsto \langle \alpha, x\rangle + c$ for some $x$; its
  minimum on $\Gamma(f_B)$ is the bundle's value at $x$, attained uniquely at $(\alpha_C, c_C)$
  iff $C$ is the unique minimizer there.

\item \emph{Nowhere unique minimizer $\Leftrightarrow$ deletion preserves the function.}
  The value of the bundle changes at $x$ iff $C$ is the unique minimizer at $x$.

\item \emph{Weakly above the lower hull $\Leftrightarrow$ not a vertex.}
  $\Gamma(f_B) = \operatorname{conv}(S' \cup \{(\alpha_C, c_C)\}) +
  \mathbb{R}_{\ge 0}e_{d+1}$. If $(\alpha_C, c_C) \in \Gamma(S')$ it is a convex
  combination of points other than itself, hence not extreme; conversely, if not extreme it lies
  in $\Gamma(S')$ because all other extreme points belong to $S'$.
\end{itemize}

\emph{Support bound.} Weakly above the lower hull means $\alpha_C \in
\operatorname{conv}(\{\alpha_i\})$ with $c_C \ge h(\alpha_C)$, where $h$ is the lower-hull
function of $S'$. By Carath\'eodory, $\alpha_C$ lies in a simplex of at most $d+1$ vertices of
$\operatorname{conv}(\{\alpha_i\})$, with rational barycentric coordinates $\lambda_i \ge 0$,
$\sum \lambda_i = 1$, and $\sum \lambda_i c_i = h(\alpha_C) \le c_C$. Clearing the common
denominator $N$ gives $k_i = N\lambda_i \in \mathbb{N}$: a certificate supported on at most
$d+1$ chains.
\end{proof}

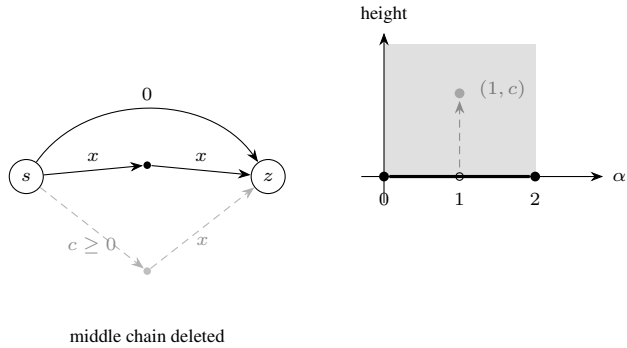
\begin{figure}[htbp]
  \centering
\begin{tikzpicture}[>=Stealth,baseline,every node/.style={font=\scriptsize},
  v/.style={circle,draw,inner sep=0pt,minimum size=4.5mm},
  c/.style={circle,fill,inner sep=1pt}]
  \node[v] (s) at (0,0) {$s$};
  \node[v] (z) at (3.2,0) {$z$};
  \draw[->] (s) to[bend left=55] node[above]{$0$} (z);
  \node[c] (b) at (1.6,0.15) {};
  \draw[->] (s)--(b) node[midway,above]{$x$}; \draw[->] (b)--(z) node[midway,above]{$x$};
  \node[c,fill=black!25] (c) at (1.6,-1.25) {};
  \draw[->,black!30,densely dashed] (s)--(c) node[midway,below,text=black!45]{$c \ge 0$};
  \draw[->,black!30,densely dashed] (c)--(z) node[midway,below,text=black!45]{$x$};
  \node at (1.6,-2.1) {middle chain deleted};
\end{tikzpicture}
\hspace{0.7cm}
\begin{tikzpicture}[>=Stealth,baseline,every node/.style={font=\scriptsize},
  pt/.style={circle,fill,inner sep=1.3pt}]
  \fill[black!12] (0,0)--(2,0)--(2,1.75)--(0,1.75)--cycle;
  \draw[->] (-0.3,0)--(2.9,0) node[right]{$\alpha$};
  \draw[->] (0,-0.35)--(0,1.9) node[above]{height};
  \foreach \a in {0,1,2} \node at (\a,-0.3) {$\a$};
  \node[pt] (p0) at (0,0) {};
  \node[pt] (p2) at (2,0) {};
  \draw[very thick] (p0)--(p2);
  \node[pt,fill=black!35] (pc) at (1,1.1) {};
  \node[right=1pt of pc,text=black!55] at (1.1,1.15) {$(1,c)$};
  \draw[->,black!45,densely dashed] (1,0.12)--(1,1.0);
  \node[circle,draw,inner sep=1pt] at (1,0) {};
  \node at (1.45,2.15) {\phantom{x}};
\end{tikzpicture}
  \caption{Mediant deletion (Lemma~\ref{lem:mediant}) on the family
  $G_c = 0 \,\|\, (x;x) \,\|\, (c;x)$, computing $\min(0,\, c+x,\, 2x)$.
  Left: the middle chain (weight $c \ge 0$ then $x$) is deleted for every
  $c \ge 0$, leaving $\min(0,2x)$. Right: the lifted points and extended Newton polyhedron.}
  \label{fig:mediant}
\end{figure}

\renewcommand{\proofname}{Proof of Theorem~\ref{thm:invariant-complete}}
\begin{proof}
\begin{itemize}
\item[(i)] \emph{Reduction.} By Lemma~\ref{lem:deshare-count}, $\Pi(G)$ vertex splits deshare
  $G$. Within each chain, Proposition~\ref{prop:list-inv-f}(1) sorts the variable edges and
  contracts the trailing constants to one edge; a $+\infty$-chain contracts to a single
  $+\infty$-edge and is deleted by~(3). Two chains sharing an exponent have equal variable
  prefixes and differ only in a final constant; merging their prefix vertices (inverse
  out-splits) leaves a parallel pair $c \| c'$, replaced by $\min(c, c')$ via~(2) (idempotency
  deletes exact duplicates). What remains is one chain per support monomial, i.e.\
  $D_{\mathrm f}(G)$. For the mediant pass, the vertices of $\Gamma(f_G)$ are fixed and every
  non-vertex chain fails Lemma~\ref{lem:mediant}, so is deletable against the vertices;
  performing all such deletions leaves $D_{\mathrm v}(G)$.

\item[(ii)] \emph{Canonicity.} $D_{\mathrm f}(G)$ is a function of $f_G$ alone
  (Definition~\ref{def:deshared}), so $D_{\mathrm f}(G) = D_{\mathrm f}(G')$ iff
  $f_G = f_{G'}$, i.e.\ $\approx_{\mathrm f}$. The mediant pass keeps exactly the vertices of
  $\Gamma$, and two functions agree iff their polyhedra, hence vertex sets, agree
  (Lemma~\ref{lem:gamma-function}(ii)), giving the $\approx_{\mathrm v}$ statement and
  decidability.

\item[(iii)] \emph{Connectivity.} Each formal move is invertible (merging inverts splitting,
  adjunction inverts deletion, etc.), so the reduction path of~(i) can be reversed. Two graphs
  $G, G'$ with the same canonical form are therefore connected by
  $G \to D(G) = D(G') \to G'$; by~(ii), same canonical form is equivalent to congruence.
  Conversely each move preserves the class.
\end{itemize}
\end{proof}
\renewcommand{\proofname}{Proof}

\begin{table}[h]
  \caption{The calculus of invariant operations on DAGs derived from tropical semiring properties.}
  \label{tab:moves-axioms}
  \centering
\small
\setlength{\tabcolsep}{4pt}
\begin{tabular}{@{}p{7.6cm}p{4.6cm}@{}}
\toprule
graph operation & tropical semiring property \\
\midrule
\emph{series-run reweighting}: redistribute weights along an unbranched run,
  sum preserved
  & $+$ associative, commutative \\
\emph{parallel-bundle bookkeeping}: permute or delete/adjoin a dominated branch
  & $\min$ assoc., comm., idempotent \\
\emph{$+\infty$-edge adjunction}: add/remove $+\infty$-weighted edges
  (preserving 2-terminality)
  & $\min(a, {+}\infty) = a$, $a + \infty = +\infty$ \\
\emph{vertex splitting / merging}: split edges at an internal vertex into two
  copies, or merge
  & distributivity of $+$ over $\min$ \\
\midrule
\emph{mediant deletion}: delete a path whose monomial lies weakly above the
  lower hull of the others
  & redundant monomials may be deleted by convexity \\
\bottomrule
\end{tabular}
\end{table}

\section{Geometry of the two substitutions}\label{app:geometry}

\begin{definition}\label{def:coord-deletion}
Write $\pi_i : \mathbb{R}^d \to \mathbb{R}^{d-1}$ for deletion of coordinate $i$, $\alpha_{-i} = \pi_i(\alpha)$,
and $\hat\pi_i = \pi_i \times \mathrm{id}_{\mathbb{R}}$ for its lift keeping the height. The shadow uses
$\hat\pi_i(\Gamma(f))$, which retains coefficients; $\pi_i(\mathcal{N}(f))$ forgets them.
\end{definition}

\begin{definition}\label{def:rec-complex}
For a polyhedral complex $\mathcal{C}$ and $v \ne 0$, $\operatorname{rec}_v(\mathcal{C}) = \{C \in \mathcal{C} : v \in
\operatorname{rec}(C)\}$ is the subcomplex of cells stable under translation by $v$
\citep[\S A]{joswig2021essentials}.
\end{definition}

\renewcommand{\proofname}{Proof of Proposition~\ref{prop:image-Pi}}
\begin{proof}
$P_i$ acts by a single transformation across the three languages by the triality
(Theorem~\ref{thm:triality}): the rule ``set $x_i = +\infty$'' respects $\approx_{\mathrm f}$, hence
descends to classes and has one avatar in each language.

\emph{Algebra.} With $\hat S = \hat S(f_G)$, the sign of $\alpha_i$ partitions
$\hat S = \hat S^{>}_i \sqcup \hat S^{0}_i$ ($\alpha_i > 0$ resp.\ $=0$). Since $+\infty$ is
absorbing for $+$, setting $x_i = +\infty$ kills every monomial with $\alpha_i \ge 1$ and fixes
the rest with coefficients, so $\hat S(P_i(f_G)) = \hat S^{0}_i$; writing $h$ for the
$\hat S^{>}_i$ block, $f_G = \min(h, P_i(f_G))$.

\emph{Geometry.} The exponents lie in $\mathbb{Z}^{t+1}_{\ge 0}$ and the upward ray fixes coordinate
$i$, so $\Gamma(f_G) \subseteq \{\alpha_i \ge 0\}$; hence $\{\alpha_i = 0\}$ is supporting and
$\Gamma(f_G) \cap \{\alpha_i = 0\}$ is a (possibly empty) face, carrying exactly $\hat S^{0}_i$.
For the hypersurface, take $f_G$ in function form. A cell of $\mathcal{T}(f_G)$ has an argmin set
$\sigma$; those with $\sigma \subseteq \hat S^{0}_i$ are stable under increasing $x_i$ (no
$x_i$ appears), while any $\sigma$ containing a monomial with $\alpha_i > 0$ is eventually
dominated by a surviving one, so leaves the cell. Thus the cells stable in direction $e_i$ are
exactly those on $\hat S^{0}_i$, i.e.\ $\operatorname{rec}_{e_i}(\mathcal{T}(f_G))$; projecting by
$\pi_i$ gives $\mathcal{T}(P_i(f_G))$, provided $\hat S^{0}_i \ne \emptyset$ (else $P_i(f_G) = +\infty$).

\emph{Graph.} The $s$--$z$ paths through an $x_i$-edge acquire weight $+\infty$ and drop out of the
$\min$; those avoiding $E_i$ are the paths of $G \setminus E_i$ with unchanged weight, so
$f_G|_{x_i = +\infty} = f_{G \setminus E_i}$.

\medskip
\emph{The $Q_i$ column.} $Q_i$ is evaluation at $x_i = 0$, a semiring homomorphism, so it acts
term by term: the term at $\alpha$ becomes $c_\alpha + x_{-i}^{\alpha_{-i}}$, each $x_i$
contributing the $+$-identity $0$.

\emph{Algebra.} In lifted coordinates $Q_i(f_G)$ is the $\min$ over the projected set
$A = \hat\pi_i(\hat S(f_G))$; as $\pi_i$ need not be injective on the support, $A$ may carry
several points over one exponent, and normal form keeps the least:
\[
  \hat S(Q_i(f_G)) = \vmin(A),
  \quad\text{support } \pi_i(S(f_G)),
  \quad\text{coefficient } \min\{c_\alpha : \pi_i(\alpha) = \beta\}.
\]

\emph{Geometry.} Adding the upward ray $R$ absorbs the points $\vmin$ discards, each lying on
$R$ above a retained one, so
\[
  \Gamma(Q_i(f_G)) = \operatorname{conv}(A) + R = \hat\pi_i(\Gamma(f_G)),
\]
since $\hat\pi_i$ is linear and fixes $R$.

\emph{Hypersurface.} As $Q_i(f_G)$ is $f_G$ restricted to $\{x_i = 0\}$, a tie of $Q_i(f_G)$ is
one of $f_G$ on that slice and conversely, save that two monomials sharing a $\pi_i$-image agree
on the whole slice and merge into one; their common wall fills a full-dimensional cell on which a
single monomial of $Q_i(f_G)$ dominates, and removing those relative interiors is the stated
correction.

\emph{Graph.} Reweighting $E_i$ to $0$ substitutes $x_i = 0$ into every path monomial, and the
two normalizing moves are invariant; this is reweighting rather than contraction, which can
create or destroy an $s$--$z$ path and so misrepresents $Q_i$.

\medskip
\emph{$Q_i \le P_i$.} The coordinate face lies in the coordinate shadow:
\[
  \Gamma(P_i(f_G))
  = \hat\pi_i\!\bigl(\Gamma(f_G) \cap \{\alpha_i = 0\}\bigr)
  \subseteq \hat\pi_i\!\bigl(\Gamma(f_G)\bigr)
  = \Gamma(Q_i(f_G)),
\]
and since tropical addition of functions is the convex hull of the union of their polyhedra,
$\min(Q_i(f_G), P_i(f_G)) = Q_i(f_G)$, i.e.\ $Q_i(f_G) \le P_i(f_G)$.
\end{proof}
\renewcommand{\proofname}{Proof}

\begin{corollary}\label{cor:fiber-Pi}
For a tropical polynomial $g$ in the variables $\{x_j : j \ne i\}$,
\[
  P_i^{-1}(g) = \bigl\{\, g \oplus h \;:\; \text{every monomial of } h \text{ has } \alpha_i \ge 1 \,\bigr\},
\]
\[
  Q_i^{-1}(g) = \Bigl\{\, \bigoplus_{\beta \in S(g)} c_\beta^{g} \odot x_{-i}^{\odot\beta} \odot x_i^{\odot k_\beta} \;\oplus\; r \;:\;
    k_\beta \in \mathbb{N},\; \pi_i(S(r)) \subseteq S(g),\;
      c_\alpha^{r} \ge c_{\pi_i(\alpha)}^{g} \,\Bigr\}.
\]
The shapes are opposite: $P_i$ pins the $\alpha_i = 0$ layer and frees everything above position
$i$, while $Q_i$ frees the exponent $k_\beta$ at position $i$ and constrains everything above
the fiber minimum.
\end{corollary}

\renewcommand{\proofname}{Proof of Corollary~\ref{cor:fiber-Pi}}
\begin{proof}
\emph{$P_i$-fiber.} By the algebraic clause, every $f = \min(g, h)$ with each monomial of $h$
carrying $\alpha_i \ge 1$ has $\hat S(f) = \hat S(g) \sqcup \hat S(h)$ and $P_i(f) = g$.
Conversely any $f$ not of that form either omits a monomial of $g$, forcing
$S(P_i(f)) \subsetneq S(g)$, or contributes a monomial with $\alpha_i = 0$ beyond $g$,
forcing $S(g) \subsetneq S(P_i(f))$.

\emph{$Q_i$-fiber.} $Q_i$ is a homomorphism, so any
\[
  g_k = \min_\beta \bigl(c_\beta^{g} + x_{-i}^{\beta} + k_\beta x_i\bigr)
\]
has $Q_i(g_k) = g$, since each $k_\beta x_i$ evaluates to $0$. For $f = \min(g_k, r)$,
\[
  Q_i(f) = \min(g,\, Q_i(r)) = g
  \quad\iff\quad
  Q_i(r) \ge g \text{ pointwise};
\]
by the algebraic clause that is $\pi_i(S(r)) \subseteq S(g)$ together with
$c_\alpha^{r} \ge c_{\pi_i(\alpha)}^{g}$ at every monomial of $r$.
\end{proof}
\renewcommand{\proofname}{Proof}

\section{Edge substitutions and the series-parallel boundary}\label{app:sp-boundary}

\begin{definition}\label{def:edge-sub}
For a two-terminal DAG $H$ with $r$ labeled edges, the \emph{edge substitution} $\Phi_H$ sends
two-terminal DAGs $A_1, \dots, A_r$ to the two-terminal DAG obtained by replacing edge~$i$ of
$H$ with~$A_i$, identifying the edge's endpoints with $A_i$'s source and sink.
Series composition is the edge substitution into a two-edge path; parallel composition is the
edge substitution into a pair of parallel edges.
\end{definition}

\begin{definition}\label{def:sp}
A two-terminal DAG is \emph{series-parallel} (SP) if built from a single $s \to z$ edge by
series and parallel compositions. The \emph{Wheatstone graph} $W$ is $K_4$ minus an edge, the
smallest non-SP two-terminal DAG.
\end{definition}

\begin{theorem}[\citealp{duffin1965-sp,valdes1982-sp-digraphs}]\label{thm:sp-equiv}
A two-terminal DAG is SP iff it contains no subgraph homeomorphic to $W$.
\end{theorem}

\renewcommand{\proofname}{Proof of Proposition~\ref{prop:wheatstone-glue}}
\begin{proof}
Specializing each $A_i$ to a single edge, a series-parallel composite of
single edges is SP (Definition~\ref{def:sp}); but $\Phi_W$ on five single edges is $W$, non-SP
by Theorem~\ref{thm:sp-equiv}.
\end{proof}
\renewcommand{\proofname}{Proof}

\end{document}